\documentclass[letterpaper, 10 pt, conference]{ieeeconf}  
\IEEEoverridecommandlockouts                              

\usepackage{verbatim}
\usepackage{amssymb}
\usepackage{mathtools}
\usepackage{nth}
\usepackage{cite}

\usepackage[vlined,ruled]{algorithm2e}

\usepackage{amsmath,graphicx,epsfig,color,amsfonts, subfigure}

\graphicspath{{./fig/}}

\def\expt{\mathbb{E}}
\def\real{\mathbb{R}}

\def\natural{\mathbb{N}}
\newcommand{\until}[1]{\{1,\dots, #1\}}
\newcommand{\subscr}[2]{#1_{\textup{#2}}}
\newcommand{\supscr}[2]{#1^{\textup{#2}}}
\newcommand{\setdef}[2]{\{#1 \; | \; #2\}}
\newcommand{\seqdef}[2]{\{#1\}_{#2}}

\newcommand{\union}{\operatorname{\cup}}

\newcommand\oprocendsymbol{\hbox{$\square$}}
\newcommand\oprocend{\relax\ifmmode\else\unskip\hfill\fi\oprocendsymbol}

\newcommand\bit[1]{\textit{\textbf{#1}}}

\def \bs {\boldsymbol}

\def \etal {\emph{et al.}}

\newtheorem{theorem}{Theorem}

\newtheorem{remark}{Remark}

\renewcommand{\theenumi}{(\roman{enumi}}

\title{On  Abruptly-Changing and Slowly-Varying Multiarmed Bandit Problems
\thanks{This work has been supported by NSF Award IIS-1734272.}
}

\author{Lai Wei \hspace{1in} Vaibhav Srivastava
\thanks{L. Wei and V. Srivastava are with the Department of Electrical and Computer Engineering. Michigan State University, East Lansing, MI 48823 USA.
        {\tt\small e-mail: weilai1@msu.edu; e-mail: vaibhav@egr.msu.edu }}%
}

\begin{document}

\maketitle
\thispagestyle{empty}
\pagestyle{empty}

\begin{abstract}
We study the non-stationary stochastic multiarmed bandit (MAB) problem and propose two generic algorithms, namely, the Limited Memory Deterministic Sequencing of Exploration and Exploitation (LM-DSEE) and the Sliding-Window Upper Confidence Bound\# (SW-UCB\#).
We  rigorously analyze these algorithms in abruptly-changing and slowly-varying environments and characterize their performance.
We show that the expected cumulative regret for these algorithms under either of the environments is upper bounded by sublinear functions of time, i.e., the time average of the regret asymptotically converges to zero. We complement our analytic results with numerical illustrations. 
\end{abstract}

\section{Introduction}

Decision-making in uncertain and non-stationary environments is one of the most fundamental problems across scientific disciplines, including economics, social science, neuroscience, and ecology. These problems often require balancing several decision-making tradeoffs, such as speed-versus-accuracy, robustness-versus-efficiency, and explore-versus-exploit. The MAB problem is a prototypical example of the explore-versus-exploit tradeoff: choosing between the most informative and seemingly the most rewarding alternative.

 In an MAB problem,  a decision-maker sequentially allocates a single resource by repeatedly choosing one among a set of competing alternative arms (options). These problems have been applied in several interesting areas such as robotic foraging and surveillance~\cite{JRK-AK-PT:78,VS-PR-NEL:13, VS-PR-NEL:14}, acoustic relay positioning for underwater communication~\cite{MYC-JL-FSH:13}, and channel allocation in communication networks~\cite{anandkumar2011distributed}. In a standard MAB problem, a stationary environment is considered, however, many application areas are inherently non-stationary. In this paper, we seek to address this gap and  study the MAB problem in two classes of non-stationary environments: (i) abruptly-changing environment and (ii) slowly-varying environment. 
 
The performance of a sequential allocation policy for the MAB problem is characterized in terms of the expected cumulative regret which is defined as the cumulative sum of the difference between the maximum mean reward and the mean reward at the arm selected by the policy at each time. In their seminal work, Lai and Robbins~\cite{TLL-HR:85} established a logarithmic lower bound on the expected cumulative regret incurred by any policy in a stationary MAB problem and proposed an algorithm that achieves this lower bound. Several subsequent works have focused on design of simpler algorithms that achieve the logarithmic lower bound (see~\cite{SB-NCB:12}, and references therein). In a non-stationary environment, achieving logarithmic  expected cumulative regret may not be feasible and the focus is on design of algorithms that achieve sublinear expected cumulative regret, i.e., the time average of the regret asymptotically converges to zero.  

Some classes of non-stationary MAB have been studied in the literature. In~\cite{auer2002nonstochastic}, authors study a non-stochastic MAB problem in which the rewards are deterministic and non-stationary. They study a weaker notion of the regret, wherein the policy generated by the algorithm is compared against the best policy within the policies that select the same arm at each time. In a recent work~\cite{besbes2014optimal}, the algorithms developed in~\cite{auer2002nonstochastic} are adapted to handle a class of non-stationary environments and upper bounds on the standard notion of the regret are derived.  In~\cite{AG-EM:08}, authors study a class of non-stationary MAB problems in which the mean rewards at any arm may switch abruptly at unknown times to some unknown values. They design an upper confidence bound (UCB) based algorithm that relies on estimates of mean rewards from a recent time-window of observations. In~\cite{liu2017change}, authors study the MAB problem in a piecewise-stationary environment. They use active detection algorithms to determine the change-points and restart the UCB algorithm.  Some other examples of non-stationary MAB problems are discussed in~\cite{burtini2015survey, gupta2011thompson,raj2017taming}.


An application area of interest for the MAB problem is robotic search and surveillance in which a robot is routed to collect stochastic rewards~\cite{dimitrova2017robot, VS-FP-FB:11za}. These rewards may correspond to, for example,  likelihood of an anomaly at a spatial location, concentration of a certain type of algae in the ocean, etc. MAB Algorithms have been extended to these problems by introducing block-allocation strategies that seek to  balance the explore-exploit tradeoff using sufficiently small travel time~\cite{RA-MVH-DT:88,PR-VS-NEL:13d}. In~\cite{VS-PR-NEL:14}, authors developed block-allocation strategies for the MAB problem with abruptly-changing reward by extending the algorithm proposed in~\cite{AG-EM:08}. 

While the above algorithms balance the explore-exploit tradeoff while ensuring sufficiently small travel time, they are reactive in the sense that they select only one arm at a time, i.e., they only provide information about the next location to be visited by the robot. 
Certain motion constraints on the robots such as non-holonomicity may make such movements energetically demanding. Therefore, we seek algorithms that have a deterministic and predictable structure that can be leveraged to design trajectories for the robot which can be efficiently traversed even under motion constraints. Towards this end, we focus on DSEE algorithms~\cite{SV-KL-QZ:13,NN-DK-RJ:16,HL-KL-QZ:13}. 

In this paper, we study the MAB problem in abruptly-changing and slowly-varying environments, and develop upper confidence bound type and DSEE type algorithms for these environments. Our assumptions on the environment are similar to those in~\cite{besbes2014optimal} and~\cite{AG-EM:08}, but we focus on alternative algorithms which include algorithms with deterministic structure as discussed above. In particular, we extend the DSEE algorithm to non-stationary environments and develop the LM-DSEE algorithm. We also extend the SW-UCB algorithm, developed and analyzed for abruptly-changing environments in~\cite{AG-EM:08}, to  the SW-UCB\# algorithm for non-stationary environments. 
%
A sliding-window of observations is used in the SW-UCB algorithm and tuning of the fixed width of this sliding-window requires the knowledge of the the horizon length of the problem. In the SW-UCB\# algorithm, we relax this requirement by considering a time-varying length of the sliding-window.

The major contributions of this paper are threefold. First, we develop two novel algorithms: the LM-DSEE and the SW-UCB\# for the non-stationary MAB problem. Second, we analyze the LM-DSEE and the SW-UCB\# algorithms for abruptly-changing and slowly-varying environments and establish upper bounds on the expected cumulative regret. 
Third, we illustrate our analytic results with numerical examples. 

The remainder of the paper is organized as follows. In Section~\ref{sec:background}, we present some background material and formally introduce the non-stationary stochastic MAB problem. We present two novel algorithms, the LM-DSEE and the SW-UCB\# in Section~\ref{sec:algorithms}. In Sections~\ref{sec:LM-DSEE} and \ref{sec:SW-UCB}, we,  respectively, analyze the LM-DSEE and the SW-UCB\# algorithms for abruptly-changing and slowly-varying environments. We illustrate our analytic results with numerical examples in Section~\ref{sec:simulation}, and conclude in Section~\ref{sec:conclusions}.

\section{Background \& Problem Description}\label{sec:background}

In this section, we recall the stochastic MAB problem in a stationary environment, and introduce the stochastic MAB problem in two classes of  non-stationary environments, namely, the abruptly-changing environment and the slowly-varying environment.


\subsection{The stationary stochastic MAB problem}
Consider an $N$-armed bandit problem, i.e., an MAB problem with $N$ arms. 
The reward associated with arm $j\in \until{N}$ is a random variable with bounded support $[0,1]$ and an unknown stationary mean $\mu_j \in [0,1]$. Let the decision-making agent choose arm $j_t$ at time $t \in \until{T}$ and receive a reward $r_t$ associated with the arm.  The decision-maker's objective  is to choose a sequence of arms $\seqdef{j_t}{t\in \until{T}}$ that maximizes the expected cumulative reward $\sum_{t=1}^T \mu_{j_t}$, where $T$ is the horizon length of the sequential allocation process.

For an MAB problem, the expected \emph{regret} at time $t$ is defined by $\mu_{j^*}-\mu_{j_t}$, where $\mu_{j^*} = \max \setdef{\mu_j}{j\in\until{N}}$. The objective of the decision-maker can be equivalently defined as minimizing the expected cumulative regret defined by $ R(T) = \sum_{t=1}^T \expt[\mu_{j^*}-\mu_{j_t}]= \sum_{j=1}^N \Delta_j \expt[n_{j}(T)]$, 
where $n_{j}(T)$ is the cumulative number of times a suboptimal arm $j$ has been chosen until time $T$ and $\Delta_j = \mu_{j^*}-\mu_j$ is the expected regret due to picking arm $j$ instead of arm $j^*$.

\subsection{Algorithms for the stationary stochastic MAB problem}
We recall two state-of-the-art algorithms for the stationary stochastic MAB problem relevant to this paper: (i) the UCB algorithm, and (ii) the DSEE algorithm. 

The UCB algorithm maintains a statistical estimate of the mean rewards associated with each arm. It initializes by selecting each arm once and subsequently selects the arm $j_t$ at time $t$ defined by 
\[
j_t \in \text{arg max} \Big\{\bar r_j (t-1) + \sqrt{\frac{2 \ln (t-1) }{n_j(t-1)}}\Big| j \in \until{N} \Big\}, 
\]
where $\bar r_j (t-1)$ is the statistical mean of the rewards received at arm $j$ until time $t$. 
Auer~\etal \ \cite{PA-NCB-PF:02} showed that the UCB algorithm achieves expected cumulative regret that is within a constant factor of the optimal. 

 The DSEE algorithm divides the set of natural numbers $\natural$ into interleaving epochs of exploration and exploitation~\cite{SV-KL-QZ:13}. In the exploration block each arm is played in a round-robin fashion, while in the exploitation block, only the arm with the maximum statistical mean reward is played. 
%
%
For an appropriately defined $w \in \real_{>0}$, the DSEE algorithm at time $t$ exploits if number of exploration steps until time $t-1$ are greater than or equal to $N \lceil w \log t \rceil$, otherwise it starts a new exploration block. Vakili~\etal \ \cite{SV-KL-QZ:13} derived bounds on the regret of the DSEE algorithm.

\subsection{The non-stationary stochastic MAB problem}

The non-stationary stochastic MAB problem is the stochastic MAB problem in which the mean reward at each arm is changing with time. Let the mean reward associated with arm $j$ at time $t$ be $\mu_j(t) \in [0,1]$. The decision-maker's objective  is to choose a sequence of arms $\seqdef{j_t}{t\in \until{T}}$ that maximizes the expected cumulative reward $\sum_{t=1}^T \mu_{j_t}(t)$, where $T$ is the horizon length of the sequential allocation process. We will characterize the performance of algorithms for these problems using the notion of the expected cumulative regret defined by
\begin{align*}
R(T) &= \sum_{t=1}^T \expt[\mu_{j_t^*}(t) - \mu_{j_t}(t)] \\
&=\sum_{t=1}^T \mu_{j_t^*}(t) - \expt \Big[\sum_{j=1}^{N}\sum_{t=1}^{T}\bs{1}_{\{ j_t =j \}}\mu_j(t) \Big],
\end{align*}
where $\mu_{j_t^*}(t)=\max_{j \in \until{N}} \mu_j(t)$, $\bs{1}_{\{ \centerdot \}}$ is the indicator function and the expectation is computed over different realizations of $j_t$. For brevity, in the following, we will refer to $R(T)$ simply as the regret.

In this paper, we study the above MAB problem for two classes of non-stationary environments: 

\noindent
\bit{Abruptly-changing environment:} 
In an abruptly-changing environment, the mean rewards from arms switch to unknown values at unknown time-instants. We refer to these time-instants as \emph{breakpoints}. We assume that the number of breakpoints until time $T$ is $\Upsilon_T \in O(T^\nu)$, where $\nu \in [0,1)$ and is known a priori. 

\noindent
\bit{Slowly-varying environment:}
In a slowly-varying environment, the change in the mean reward at each arm between any two subsequent time-instants is small and is upper bounded by $\epsilon_T \in O(T^{-\kappa})$, where $\kappa \in \real_{>0}$ and is known a priori. Here, lower values of $\kappa$ correspond to higher changes in the mean reward at subsequent time-instants. In the following, we will refer to $\epsilon_T$ as the non-stationarity parameter.

\section{Algorithms for Non-Stationary Stochastic MAB Problem}\label{sec:algorithms}
In this section, we present two algorithms for the non-stationary stochastic MAB problem: the Limited-Memory DSEE (LM-DSEE) algorithm and the Sliding-Window UCB\# (SW-UCB\#) algorithm.
These algorithms are generic and require some parameters to be tuned based on environment characteristics.  

\subsection{The LM-DSEE algorithm} 
The LM-DSEE algorithm comprises interleaving blocks of exploration and exploitation. In the $k$-th exploration epoch, each arm is sampled $L(k)=\lceil \gamma \ln (k^\rho l b) \rceil$ number of times.  In the $k$-th exploitation epoch, the arm with the highest sample mean in the $k$-th exploration epoch is sampled $\lceil a k^\rho l \rceil -NL(k)$ times. Here, parameters $\rho, \gamma,a,b$ and $l$ are tuned based on the environment characteristics (see Algorithm~\ref{algo:lm-dsee} for details). In the following, we will set $a$ and $b$ to unity for the purposes of analysis. Parameters $a$ and $b$ do not influence the order of regret bounds derived below, but they can be tuned to enhance the transient performance.

The LM-DSEE algorithm is similar in spirit to the DSEE algorithms~\cite{SV-KL-QZ:13,NN-DK-RJ:16}, wherein the length of exploitation epoch increases exponentially with the epoch number and all the data collected in the previous exploration epochs is used to estimate the mean rewards. However, in a non-stationary environment, using all the rewards from the previous exploration epochs may lead to a heavily biased estimate of the mean rewards. Furthermore, an exponentially increasing exploitation epoch length may lead to excessive exploitation based on an outdated estimate of the mean rewards. To address these issues, we modify the DSEE algorithm by using only the rewards from the current exploration epoch to estimate the mean rewards, and we increase the length of exploitation epoch using a power law.

\IncMargin{.3em}
\begin{algorithm}[t]
  {\footnotesize
   \SetKwInOut{Input}{  Input}
   \SetKwInOut{Set}{  Set}
   \SetKwInOut{Title}{Algorithm}
   \SetKwInOut{Require}{Require}
   \SetKwInOut{Output}{Output}
   
{\it For abruptly-changing environment} \\   
   
   \Input{$\nu \in [0,1)$,\! $\Delta_{\min} \! \in (0,1)$, $T \in \natural$,\! $a \in \real_{>0}$,\! $b \in (0,1]$\;}
   
    \Set{$\gamma \geq \frac{2}{\Delta_{\min}^2}$,
    $l\in \{\frac{N}{a} \lceil \gamma \ln l b \rceil, \ldots, +\infty\} $, and $\rho=\frac{1-\nu}{1+\nu} $\;}
    
    \smallskip

 {\it For slowly-varying environment} \\       
   
    \Input{$\kappa \in \real_{>0}$, $\kappa_{\max} \in (0,\frac{4}{3})$,\! $T \in \natural$,\! $a \in \real_{>0}$, $b \in (0,1]$;}
    
      \Set{$ \tilde \kappa \leftarrow \min\{\kappa, \kappa_{\max}\}, \rho \leftarrow \frac{3 \tilde \kappa}{4-3 \tilde \kappa},$
      \smallskip
      $l \in \{\frac{N}{a} \lceil l^\frac{2}{3} \ln l b \rceil, \ldots, +\infty\}$,   
      and  $\gamma =2(k^\rho l)^{\frac{2}{3}};$}
  
  \medskip  
      
   \Output{sequence of arm selection\;}

   \medskip
   
\emph{\% Initialization:} 
\smallskip

   \nl Set batch index $k \leftarrow 1$ and $t \leftarrow 1$\;

   \smallskip
   
   \nl \While{$t \leq T$}{
   \smallskip
   
   \emph{\% Exploration}
   
   \smallskip
   
   \nl  \For{$j \in \until{N}$ \smallskip}
   {Pick arm $j$,  $L(k) \leftarrow \lceil \gamma \ln (k^\rho l b) \rceil$  times \; \smallskip
   collect rewards $\seqdef{\hat r_j^i(k)}{i \in \until{L(k)}}$ \;\smallskip
  compute sample mean $\supscr{\bar r}{epch}_j(k) \leftarrow \frac{1}{L(k)}\sum_{i=1}^{L(k)} \hat r_j^i(k)$\;  
   }
\smallskip

   \emph{\% Exploitation}
   \smallskip
   
\nl Select the best arm   $\supscr{j}{epch}_k=\arg \max_{j \in \until{N}}\supscr{\bar r}{epch}_j(k)$ \; 

\smallskip

\nl Pick arm $\supscr{j}{epch}_k$, $\lceil a k^\rho l \rceil -NL(k)$ times \; 
\smallskip

\nl Update batch index $k \leftarrow k+1$ and $t \leftarrow t+ \lceil ak^\rho l \rceil$
   }


    \caption{\textit{The LM-DSEE Algorithm}}
  \label{algo:lm-dsee}}
\end{algorithm} 
\DecMargin{.3em}

\subsection{The SW-UCB\# algorithm} 
The SW-UCB\# algorithm is an adaptation of the SW-UCB algorithm proposed and studied in~\cite{AG-EM:08}. The SW-UCB\# algorithm, at time $t$, maintains an estimate of the mean reward $\bar{r}_j(t,\alpha)$ at each arm $j$, using only the rewards collected within a sliding-window of observations. 
Let the width of the sliding-window at time $t \in \until{T}$ be $\tau(t, \alpha) = \min\{\lceil \lambda t^\alpha \rceil, t\}$, where parameters $\alpha \in (0,1]$ and $\lambda \in \real_{\ge 0} \union \{+ \infty\}$ are tuned based on environment characteristics. 
Let $n_j(t,\alpha)=\sum_{s= t- \tau(t,\alpha)+1}^{t} \bs{1}_{\{ j_s=j \}}$ be the number of times arm $j$ has been selected in the time-window at time $t$, then 
\[
\overline{r}_j(t, \alpha)=\frac{1}{n_j(t, \alpha)} \sum_{s= t- \tau (t, \alpha)+1}^{t} r_j(s) \bs{1}_{\{ j_s=j \}}.
\]

Based on the above estimate, the SW-UCB algorithm at each time selects the arm 
\begin{equation}
j_t = \text{arg max} \setdef{\overline{r}_j(t-1, \alpha) + c_j(t-1, \alpha)}{j \in \until{N}},\label{ucb calculation}
\end{equation}
where  $c_j(t, \alpha)=  \sqrt{\frac{(1+\alpha) \ln t}{n_j(t, \alpha)}}$.
The details of the algorithm are presented in Algorithm 2.

In contrast to the SW-UCB algorithm~\cite{AG-EM:08}, the SW-UCB\# algorithm employs a time-varying width of the sliding-window. The tuning of the fixed window width in~\cite{AG-EM:08} requires a priori knowledge of the time horizon $T$ which is no longer needed for the SW-UCB\# algorithm. 

\IncMargin{.3em}
\begin{algorithm}[t]
	{\footnotesize
		\SetKwInOut{Input}{  Input}
		\SetKwInOut{Set}{  Set}
		\SetKwInOut{Title}{Algorithm}
		\SetKwInOut{Require}{Require}
		\SetKwInOut{Output}{Output}
		
		{\it For abruptly-changing environment} \\   
		
		\Input{$\nu \in [0,1)$, $\Delta_{\min} \in (0,1)$, $\lambda \in \real_{>0}$ \& $T\in \natural$\;}
		
		\Set{$\alpha = \frac{1-\nu}{2}$}
		\smallskip
		
		{\it For slowly-varying environment} \\     
		 
		\Input{$\kappa \in \real_{>0}$, $\lambda \in \real_{>0}$ \& $T\in \natural$\;}
		\Set {$\alpha= \min\{1,\frac{3\kappa}{4}\}$}
		
		\medskip		
%
		
		\Output{sequence of arm selection\;}
		
		\medskip
		
		\emph{\% Initialization:} 
		\smallskip
		
		 \nl \While{$t \leq T$}{
			\smallskip
				
			\smallskip
			
			\nl  \If{$t \in \until{N}$ \smallskip}
			{Pick arm $j_t=t$;
			}
		
			\smallskip
						
			\smallskip
			
			\nl  \Else
			{Pick arm $j_t$ defined in \eqref{ucb calculation} \;
			}
			\smallskip					
		}

		
		\caption{\textit{The SW-UCB\# Algorithm}}
		\label{algo:sw-ucb}}
\end{algorithm} 
\DecMargin{.3em}

%
%
%
%

\section{Analysis of the LM-DSEE algorithm} \label{sec:LM-DSEE}

In this section, we analyze the performance of the LM-DSEE algorithm (Algorithm~\ref{algo:lm-dsee}) in abruptly-changing and slowly-varying environments. 

\subsection{LM-DSEE in the abruptly-changing environment}
Before we analyze the LM-DSEE algorithm in the abruptly-changing environment, we introduce the following notations. Let
\begin{align*}
\Delta_j &=\max \setdef{\mu_{j^*_t}(t) - \mu_j(t)}{t \in \until{T}},\\
\Delta_{\max} &=  \max \setdef{\Delta_j}{j \in \until{N}},\\
\text{and }\Delta_{\min} &=\min  \{ \mu_{j^*_t}(t)-\mu_j(t) \;|\; \\
& \qquad \qquad t \in \until{T}, j \in \until{N}  \setminus \{j^*_t\} \}.
\end{align*}

\begin{theorem}[\bit{LM-DSEE in  abruptly-changing environment}]\label{thm:LMDSEE-abrupt}
For the abruptly-changing environment with the number of breakpoints $\Upsilon_T \in O(T^\nu)$, $\nu \in [0,1)$ and the LM-DSEE algorithm, the expected cumulative regret 
\[
R^{\text{LM-DSEE}}(T)\in O(T^{\frac{1+\nu}{2}}\ln T).
\]  
\end{theorem}
\medskip

\begin{proof}
Let $K$ be the index of the epoch containing the time-instant $T$, then 
the length of each epoch  is at most $\lceil {K}^\rho l \rceil$.  
Since breakpoints are located in at most $\Upsilon_T$ epochs, we can upper bound the regret from epochs containing breakpoints by
$$R_b \leq \Upsilon_T \lceil {K}^\rho l \rceil \Delta_{\max}.$$

In the epochs containing no breakpoint, let $R_e$ and $R_i$ denote, respectively, the regret from exploration and exploitation epochs.  Note that in such epochs the mean reward from each arm does not change. We denote the best arm in the $k$-th epoch with no breakpoint by $\subscr{j}{no-break}^*(k)$ and its mean by $\subscr{\mu}{no-break}^*(k)$.


Then, the regret in exploration epochs $R_e$ satisfies,
\[
R_e \leq \sum_{k=1}^{K} \sum_{j=1}^{N} \lceil \gamma \ln (k^\rho l) \rceil \Delta_j \leq K\lceil \gamma \ln (K^\rho l) \rceil \sum_{j=1}^{N} \Delta_j.
\]
In exploitation epochs, regret is incurred if the best arm is not selected, and consequently the regret in exploitation epochs $R_i$ satisfies
\begin{equation}\label{ab_dsee_ri}
 R_i \leq \sum_{k=1}^{K}\sum_{j=1}^{N} \big[ \lceil  k^\rho l \rceil-NL(k) \big]  \mathbb{P}(\supscr{j}{epch}_k = j \neq \subscr{j}{no-break}^*(k)) \Delta_j,
\end{equation}
where $\supscr{j}{epch}_k$ is the arm selected in the $k$-th exploitation epoch and $L(k)$ is the number of times an arm is selected in the $k$-th exploration epoch.

It follows from the Chernoff-Hoeffding inequality~\cite[Theorem 1]{WH:63} that 
\begin{align*}
\mathbb{P}(\supscr{\overline{r}}{epch}_j(k) &\geq \supscr{\mu}{epch}_j(k)+\delta)\\
&=\mathbb{P}(\supscr{\overline{r}}{epch}_j(k) \le \supscr{\mu}{epch}_j(k)-\delta) \leq \exp (-2\delta^2 L(k)),
\end{align*}
where $\supscr{\mu}{epch}_j(k)$ is the mean reward of arm $j$ in the $k$-th epoch. Thus, we have
\begin{equation*}
 \begin{split}
&\mathbb{P}\big(\supscr{j}{epch}_k=j \neq  \subscr{j}{no-break}^*(k)\big) \\
& \leq \mathbb{P}\big(\supscr{\overline{r}}{epch}_j(k) \geq \supscr{\mu}{epch}_j(k) +\frac{\Delta_{\min}}{2}\big)\\
&+\mathbb{P}\big(\supscr{\overline{r}}{epch}_{\subscr{j}{no-break}^*} (k) \leq 
\subscr{\mu}{no-break}^*(k)-\frac{\Delta_{\min}}{2}\big) \\
&\leq 2 \exp\big(-\frac{\Delta_{\min}^2}{2} \gamma \ln (k^\rho l) \big).
 \end{split}
\end{equation*}
Since $\gamma \geq \frac{2}{\Delta_{\min}^2}$, $\mathbb{P}(\supscr{j}{epch}_k=j \neq \subscr{j}{no-break}^*(k)) \leq 2(k^\rho l)^{-1}$.
 Substituting it into (\ref{ab_dsee_ri}), we have $R_i \leq 2K\sum_{j=1}^{N} \Delta_j$ since $\big[ \lceil  k^\rho l \rceil-NL(k) \big]< k^\rho l$. Furthermore, using the fact that $\sum_{k=1}^{K}k^\rho$ can be upper-bounded and lower-bounded by areas under the curves $f(k)=k^\rho$ and $f(k)=(k+1)^\rho$ respectively, we have
\[
\frac{l}{1+\rho}(K-1)^{1+\rho}-K\leq T \leq \frac{l}{1+\rho}(K+1)^{1+\rho}+K,
\] 
and consequently $K \in O(T^{\frac{1}{1+\rho}})$.
Therefore, it follows that
\begin{align*}
R^{\text{LM-DSEE}}(T) &=R_b+R_e+R_i \\
&\leq \Upsilon_T{K}^\rho l\Delta_{\max}
+K(\lceil \gamma \ln (K^\rho l) \rceil +2) \sum_{j=1}^{N} \Delta_j.
\end{align*}
Thus, the regret $R^{\text{LM-DSEE}}(T) \in O(T^{\frac{1+\nu}{2}}\ln T)$, and this establishes the theorem.
\end{proof}

\subsection{LM-DSEE in the slowly-varying environment}

\begin{theorem}[\bit{LM-DSEE in slowly-varying environment}]\label{thm:LM-DSEE-slow}
 For the slowly-varying environment with the non-stationarity parameter $\epsilon_T=O(T^{-\kappa})$, $\kappa \in \real_{>0}$ and the  LM-DSEE algorithm, the 
expected cumulative regret 
\[
\supscr{R}{LM-DSEE}(T)\in O(T^{\frac{3+2\rho}{3+3\rho}}\ln T),
\]
where $\rho =\frac{3 \tilde \kappa}{4-3 \tilde \kappa}$, $\tilde \kappa = \min\{\kappa, \kappa_{\max}\}$, and $\kappa_{\max} \in (0,\frac{4}{3})$.
\end{theorem} 
\begin{proof}
Similar to the proof of Theorem~\ref{thm:LMDSEE-abrupt}, we divide the regret into $R_e$ and $R_i$, the regret in the exploration epoch and the exploitation epoch, respectively. It follows that
\begin{equation*}
 \begin{split}
R_e &\leq \sum_{j=1}^{N} \sum_{k=1}^{K} \lceil \gamma \ln (k^\rho l) \rceil \Delta_j \\
&\le \sum_{j=1}^{N} \Delta_j \sum_{k=1}^{K} \big[ 2(k^\rho l)^{\frac{2}{3}}\ln (K^\rho l)+1 \big]\\
& \leq \Big[\frac{2 l^{ \frac{2}{3}}}{\frac{2}{3}\rho+1}(K+1)^{\frac{2}{3}\rho +1}\ln (K^\rho l)+K \Big]  \sum_{j=1}^{N} \Delta_j.
 \end{split}
\end{equation*}

Also, for the regret in the exploitation epoch, we have
\begin{equation}\label{eq:regret-exploit-lmdsee2}
R_i \leq \sum_{j=1}^{N}\sum_{k=1}^{K} \sum_{t \in \text{epoch } k} \!\!\!\!\!
\mathbb{P}(\supscr{j}{epch}_k=j \neq j^*_t) \big(\mu_{j^*_t}(t)-\mu_{\supscr{j}{epch}_k}(t) \big).
\end{equation}

In the context of slowly-varying environment, when the best arm switches, there exists a period around the switching instant during which the difference in the mean rewards between the best arm and the next-best arm is extremely small. Such a situation needs to be handled carefully, and towards this end we define 
\[
J(t) =\setdef{j \in \until{N}}{\mu_{j^*_t}(t) - \mu_j(t) \le \sigma}, 
\]
where we set $\sigma=(k^\rho l)^{-\frac{1}{3}}+2\varrho$ and $\varrho = \epsilon_T k^\rho l $, which is the maximum change in the mean reward at any arm in the $k$-th epoch. Then, it follows that 
\begin{align*}
\mathbb{P}(\supscr{j}{epch}_k =j \neq j_t^*)
&=\mathbb{P}(\supscr{j}{epch}_k=j \neq j_t^*, j \in J(t)) \\
&+  \mathbb{P}(\supscr{j}{epch}_k=j \neq j_t^*, j \notin J(t)).
\end{align*}
Substituting it into~\eqref{eq:regret-exploit-lmdsee2}, we obtain
\begin{align*}
 &R_i \leq \sum_{j=1}^{N}\sum_{k=1}^{K} \sum_{t \in \text{epoch } k} \big[\mathbb{P}\big(\supscr{j}{epch}_k=j, j \notin J(t)\big)\Delta_j +\sigma\big].
\end{align*}
Denote $\chi_{j,k}$ as the set of time indices at which the arm $j$ is sampled in the 
$k$-th exploration epoch. Define
$$M_j(k) \triangleq \frac{1}{|\chi_{j,k}|} \sum_{t\in \chi_{j,k}} \mu_j(t).$$
Then, it can be shown that 
\[
 \mu_{j^*_t}(t) - \mu_j(t) > \sigma \implies M_{j^*_t}(k) - M_j (k) > \sigma - 2 \varrho , 
\]
for all $t \in \text{epoch }k$. Consequently, 
 \begin{align*}
&\mathbb{P}(\supscr{j}{epch}_k=j, j\notin J(t)) \\
&\leq \mathbb{P}(\supscr{\overline{r}}{epch}_j(k)>\supscr{\overline{r}}{epch}_{j_t^*}(k) , \mu_{j^*_t}(t) - \mu_j(t) > \sigma)\\
&\leq \mathbb{P}(\supscr{\overline{r}}{epch}_j(k)>\supscr{\overline{r}}{epch}_{j_t^*}(k) , M_{j^*_t}(k) - M_j (k) > \sigma - 2 \varrho)\\
&\leq \mathbb{P} (\supscr{\overline{r}}{epch}_j(k) \geq M_j(k)+\frac{\sigma-2\varrho}{2})\\
&\quad+\mathbb{P}(\supscr{\overline{r}}{epch}_{j_t^*}(k) \leq M_{j_t^*} (k)-\frac{\sigma-2\varrho}{2}).
 \end{align*}
Since $\gamma =2(k^\rho l)^{\frac{2}{3}}$, it follows from Chernoff-Hoeffding inequality~\cite[Theorem 1]{WH:63} that
\begin{align*}
&\mathbb{P}(\supscr{j}{epch}_k=j,j\notin J(t))\\
&\leq 2 \exp(-\frac{(\sigma-2\varrho)^2}{2} \gamma \ln (k^\rho l)) \leq 2(k^\rho l)^{-1},
\end{align*}
and consequently, 
\begin{equation*}
 \begin{split}
&R_i \leq \sum_{j=1}^{N}\sum_{k=1}^{K} \big[2(k^\rho l)^{-1}\Delta_j + (k^\rho l)^{-\frac{1}{3}}+2\varrho\big]\big[ \lceil  k^\rho l \rceil-NL(k) \big] \\
&\leq \frac{N l^{\frac{2}{3}}}{\frac{2}{3}\rho+1}(K+1)^{\frac{2}{3}\rho +1}+2K \sum_{j=1}^{N}\Delta_j+\frac{2Nl^2\epsilon_T}{1+2\rho}(K+1)^{1+2\rho}.
\end{split}
\end{equation*}
Using the fact that $K \in O(T^{\frac{1}{1+\rho}})$, we have  $R_i \in O(T^{\frac{3+2\rho}{3+3\rho}})$ and $R_e \in O(T^{\frac{3+2\rho}{3+3\rho}}\ln T )$.  
Thus,  $R^{\text{LM-DSEE}}(T) \in O(T^{\frac{3+2\rho}{3+3\rho}}\ln T)$, and this concludes the proof.  
\end{proof}

\begin{remark}[\bit{The choice of $\kappa_{\max}$ in Theorem~\ref{thm:LM-DSEE-slow}}]
Theorem \ref{thm:LM-DSEE-slow} shows that the performance of the LM-DSEE improves with increasing $\rho$ and consequently, the value of $\rho$ should be chosen as high as possible. Even though the asymptotic performance of the algorithm is better at higher values of $\rho$, its finite time performance may suffer. To address this issue, we saturate the value of $\rho$ by saturating the value of $\kappa$ at $\kappa_{\max}$. The value of $\kappa_{\max}$ is a tunable parameter that dictates the tradeoff between the asymptotic and the finite time performance. \oprocend
\end{remark}

\section{Analysis of the SW-UCB\# algorithm} \label{sec:SW-UCB}
In this section, we analyze the performance of the SW-UCB\# algorithm (Algorithm~\ref{algo:sw-ucb}) in abruptly-changing and slowly-varying environments.

\subsection{SW-UCB\# in the abruptly-changing environment}

%
%
%

\begin{theorem}[\bit{SW-UCB\# in abruptly-changing environment}]\label{thm:sw-ucb-abrupt}
For the abruptly-changing environment with the number of breakpoints $\Upsilon_T=O(T^\nu)$, $\nu \in [0,1)$ and the SW-UCB\# algorithm,  the expected cumulative regret
\[
\supscr{R}{SW-UCB\#}(T) \in O(T^{\frac{1+\nu}{2}}\ln T).
\]
\end{theorem}

\begin{proof}
We define $\mathcal{T}$ such that for all $t\in\mathcal{T}$, $t$ is either a breakpoint or there exists a breakpoint in its sliding-window of observations $\{t-\tau(t-1,\alpha),\ldots,t-1\}$. For $t\in\mathcal{T}$, the estimate of the mean rewards may be significantly biased. It can be shown that
$$|\mathcal{T}| \leq \Upsilon_T \lceil \lambda (T-1)^\alpha \rceil,$$
and consequently, the regret can be upper bounded as follows:
\begin{equation}
R(T)\leq \sum_{j=1}^{N}\mathbb{E}[\tilde{N}_j(T)] \Delta_j+\Upsilon_T [\lambda (T-1)^\alpha+1]\Delta_{\max},
\label{f}
\end{equation}
where $\tilde{N}_j(T) := \sum_{t=1}^T \bs{1}_{\{ j_t=j \neq j_t^*, \, t\notin \mathcal{T}\}}$ satisfies
\begin{equation}
 \begin{split}
\tilde{N}_j(T) 
&\leq 1+\sum_{t=N+1}^T \bs{1}_{\{ j_t=j \neq j_t^*,n_j(t-1,\alpha)<A(t-1) \}} \\
& \quad +\sum_{t=N+1}^T \bs{1}_{\{ j_t=j \neq j_t^*, \, t\notin \mathcal{T},\, n_j(t-1,\alpha) \geq A(t-1) \}},
 \end{split}
 \label{a}
\end{equation}
where $A(t)=\frac{4 (1+\alpha) \ln t }{\Delta_{\min}^2}$.

We first bound the second term on the right hand side of inequality (\ref{a}). Let $G \in \natural$ be such that
\begin{equation}
[\lambda(1-\alpha)(G-1)]^{\frac{1}{1-\alpha}}< T \leq [\lambda(1-\alpha) G]^{\frac{1}{1-\alpha}}.
\label{e}
\end{equation}
Then, consider the following partition of time indices
\begin{equation}\label{eq:partition}
\big\{\{1+\lfloor [\lambda(1-\alpha)(g-1)]^{\frac{1}{1-\alpha}} \rfloor,\ldots,\lfloor [\lambda(1-\alpha) g]^{\frac{1}{1-\alpha}}\rfloor \} \big\}_{g \in \until{G}}. 
\end{equation}
In the $g$-th epoch in the partition, either
\[
\sum_{t \in g \text{-th epoch}} \bs{1}_{\{ j_t=j \neq j^*,n_j(t-1,\alpha)<A(t-1) \}}=0,
\]
or there exist at least one time-instant $t$ that $ j_t=j \neq j^*(t)$ and $n_j(t-1,\alpha)<A(t-1)$. Let the last time-instant  satisfying these conditions in the $g$-th epoch be
\begin{multline*}
t_j(g) = \max \{t \in g \text{-th epoch}|j_t=j \neq j_t^*\, \\
\text{and}\, n_j(t-1,\alpha)<A(t-1)\}.
\end{multline*}
We will now show that all but one of the time indices in the $g$-th epoch until $t_i(g)-1$ are ensured to be contained in the time-window at $t_i(g)$. Towards this end, consider the increasing convex function $f(x)=x^{\frac{1}{1-\alpha}}$ with $\alpha \in (0,1)$. It follows that $f(x_2)-f(x_1) \leq f'(x_2)(x_2-x_1)$ if $x_2 \geq x_1$. Let $\tilde{t}$  be a time index in the $g$-th epoch, and set 
$x_1= g-1$ and $x_2 = \frac{\tilde{t}^{1-\alpha}}{\lambda(1-\alpha)}$. Then, substituting $x_1$ and $x_2$ in the above inequality and simplifying, we get 
\begin{equation}\label{ineq1}
\tilde{t} - (\lambda(1-\alpha)(g-1))^{\frac{1}{1-\alpha}}\leq \lambda \tilde{t}^\alpha \Big(\frac{\tilde{t}^{1-\alpha}}{\lambda(1-\alpha)}-g+1\Big).
\end{equation}
Since by definition of the $g$-th epoch, $\frac{\tilde{t}^{1-\alpha}}{\lambda(1-\alpha)} \leq g$, we have
\begin{align}
\tilde{t}-\lfloor (\lambda (1-\alpha) (g-1))^{\frac{1}{1-\alpha}}\rfloor
&\leq \min\{ \tilde t +1, \lambda \lceil \tilde{t}^{\alpha}\rceil+1 \} \nonumber \\
& =  \tau(\tilde{t},\alpha)+1. \label{window bound 1}
\end{align}
Setting $\tilde{t} = t_j(g)-1$ in (\ref{window bound 1}), we obtain
\[
t_j(g) -\tau \big(t_j(g)-1,\alpha\big) \leq 2+\lfloor\lambda(1-\alpha)(g-1)^{\frac{1}{1-\alpha}}\rfloor,
\]
i.e., the first time-instant in the sliding-window at $ t_j(g)$ is located at or to the left of the second time-instant of the $g$-th epoch in the partition~\eqref{eq:partition}. Therefore, it follows that
\begin{align*}
&\sum_{t=1+\lfloor\lambda(1-\alpha) (g-1)^{\frac{1}{1-\alpha}}\rfloor}^{\lfloor \lambda(1-\alpha) g^{\frac{1}{1-\alpha}} \rfloor}\bs{1}_{\{ j_t=j \neq j^*,n_j(t-1,\alpha)<A(t-1) \}}\\
&\qquad \leq n_j(t_j(g)-1,\alpha)+2 \leq A(t_j(g)-1)+2.
\end{align*}
Now we have
\begin{align}\label{bound1}
&\sum_{t=N+1}^T \bs{1}_{\{ j_t=j \neq j^*,n_j(t-1,\alpha)<A(t-1) \}} \nonumber\\
&\leq 2 G+\sum_{g=1}^{G} A(t_j(g)-1) \leq G\Big(2+\frac{4 (1+\alpha) \ln T }{\Delta_{\min}^2}\Big).
\end{align}

Next, we upper-bound expectation of the last term in (\ref{a}). It can be shown that
\begin{align*}
&\bs{1}_{\{ j_t=j \neq j_t^*, \, t\notin \mathcal{T},\, n_j(t-1,\alpha) \geq A(t-1) \}}\\
&\leq \sum_{s_{j_t^*}=1}^{ \lceil \lambda (t-1)^ \alpha \rceil}\,\sum_{s_j=A(t-1)}^{\lceil \lambda (t-1)^ \alpha \rceil}\bs{1}_{\{n_j(t-1,\alpha)=s_j, \, n_{j_t^*}(t-1,\alpha)=s_{j_t^*} \}}\\ 
& \times \bs 1_{\{\overline{r}_{j}(t-1,\alpha) +c_j(t-1,\alpha) \geq \overline{r}_{j_t^*}(t-1,\alpha) +c_{j_t^*}(t-1,\alpha), \,t\notin \mathcal{T}\} }.
\end{align*}
When $t\notin \mathcal{T}$, for each $j \in \until{N}$, $\mu_j(s)$ is a constant for all $s \in \{t-\tau(t-1,\alpha),\ldots,t\}$. Note that $\overline{r}_{j}(t-1,\alpha) +c_j(t-1,\alpha)>\overline{r}_{j_t^*}(t-1,\alpha) +c_{j_t^*}(t-1,\alpha)$ means at least one of the following holds. 
\begin{align}
&\overline{r}_{j}(t-1,\alpha) \geq  \mu_j(t) + c_j(t-1,\alpha), \label{b}\\
&\overline{r}_{j_t^*}(t-1,\alpha) \leq  \mu_{j_t^*}(t) - c_{j_t^*}(t-1,\alpha), \label{c}\\
&\mu_{j_t^*} (t)- \mu_j(t)< 2 c_j(t-1,\alpha). \label{d}
\end{align}
Since $n_j(t-1,\alpha) \geq A(t-1)$, (\ref{d}) does not hold. Applying Chernoff-Hoeffding inequality~\cite[Theorem 1]{WH:63} to bound the probability of events (\ref{b}) and (\ref{c}), we obtain
\begin{align*}
&\mathbb{P}( \overline{r}_{j}(t-1,\alpha) \geq  \mu_j(t) + c_j(t-1,\alpha) )\leq (t-1)^{-2(1+\alpha)}\\
&\mathbb{P}( \overline{r}_{j^*_t}(t-1,\alpha) \leq  \mu_{j^*_t}(t) -  c_{j_t^*}(t-1,\alpha) )\leq (t-1)^{-2(1+\alpha)}.
\end{align*}
Now, we have
\begin{align}
&\mathbb{E} \Big[ \sum_{t=N+1}^T \bs{1}_{\{ j_t=j \neq j^*, \, t\notin \mathcal{T},\, n_j(t-1,\alpha) \geq A(t-1) \}} \Big]  \nonumber \\
&\leq \sum_{t=N+1}^{T}2(t-1)^{-2(1+\alpha)} [\lambda (t-1)^ \alpha+1]^2  \nonumber \\
&\leq \sum_{t=1}^{\infty}2(\lambda+1)^2 t^{-2}=\frac{(\lambda+1)^2\pi^2}{3}.
\label{i}
\end{align}
Therefore, it follows from (\ref{f}), (\ref{a}), (\ref{bound1}), and (\ref{i}) that
\begin{align*}
R(T)\leq &\sum_{j=1}^{N} \Big(G \big(2+\frac{4 (1+\alpha
) \ln T }{\Delta_{\min}^2})+1+\frac{(\lambda+1)^2\pi^2}{3}\big) \Delta_j\\
&+\Upsilon_T \lceil \lambda (T-1)^\alpha \rceil \Big)\Delta_{\max}.
\end{align*}
From (\ref{e}), we have $G=O(T^{1-\alpha})$, and this yields $R^{\text{SW-UCB\#}}(T) \in O(T^{\frac{1+\nu}{2}}\ln T)$.
\end{proof}

\subsection{The SW-UCB\# in the slowly-varying environment}

\begin{theorem}[\bit{SW-UCB\# in slowly-varying environment}]\label{thm:sw-ucb-slow}
For the slowly-varying environment with the non-stationarity  parameter $\epsilon_T =O(T^{-\kappa})$ with $\kappa \in \real_{>0}$ and the SW-UCB algorithm, the expected cumulative regret 
\[
\supscr{R}{SW-UCB\#}(T)\in O(T^{1-{\frac{\alpha}{3}}}\ln T),
\]
where $\alpha = \min{\{ 1,\frac{3\kappa}{4}\}}$.
\end{theorem}
\begin{proof}
We start by noting that the number of times arm $j$ is selected when it is suboptimal satisfies
\begin{equation}
\begin{split}
\hat{n}_j(T) 
& \le  1+\sum_{t=N+1}^T \bs{1}_{\{ j_t=j \neq j_t^*,n_j(t-1,\alpha)<A(t-1) \}} \\
& \quad +\sum_{t=N+1}^T \bs{1}_{\{ j_t=j \neq j_t^*,n_j(t-1,\alpha) \geq A(t-1),  j \notin J(t) \}}\\
& \quad +\sum_{t=N+1}^T \bs{1}_{\{ j_t=j \neq j_t^*,n_j(t-1,\alpha) \geq A(t-1), j \in J(t) \},}
\label{g}
 \end{split}
\end{equation}
 where $\sigma=t^{-\frac{\alpha}{3}}+2\lceil \lambda (t-1)^\alpha \rceil \epsilon_T$, $A(t)=4 t^{\frac{2\alpha}{3}} (1+\alpha) \ln t$, $n_j(t-1, \alpha)$ is defined in Algorithm~\ref{algo:sw-ucb}, and $J(t) =\setdef{j \in \until{N}}{\mu_{j^*_t}(t) - \mu_j(t) \le \sigma}$. 
 
 We first focus on the second term on the right hand side of~\eqref{g}, and bound it similarly to the proof of Theorem~\ref{thm:sw-ucb-abrupt}:
\begin{align}
&\sum_{t=N+1}^T \bs{1}_{\{ j_t=j \neq j^*_t,n_j(t-1,\alpha)<A(t-1) \}} \!\leq\! \sum_{g=1}^{G} [A(t_j(g)-1)+2] 
&\nonumber \\
&\leq 2G +4 (1+\alpha) \ln T \sum_{g=1}^{G} [\lambda(1-\alpha) g]^{\frac{2\alpha}{3-3\alpha}} \nonumber\\
&\leq \frac{12(1+\alpha)}{3-\alpha}\lambda^{\frac{2\alpha}{3-3\alpha}} [(1-\alpha)(G+1)]^{\frac{3-\alpha}{3-3\alpha}}\ln T + 2G, \label{eq:second-term}
\end{align}
where $G$ is defined in~\eqref{e}. 

We now analyze the third term in~\eqref{g}. Let $M_j(t)=\frac{1}{n_j(t,\alpha)} \sum_{s= t- \lceil \lambda t^\alpha \rceil+1}^t \mu_j(s) \bs{1}_{\{ j_s=j \}}$. Then, it follows similarly to the proof of Theorem~\ref{thm:sw-ucb-abrupt} that 
\begin{align*}
&\bs{1}_{\{ j_t=j \neq j^*_t, \, j \notin J(t),\, n_j(t-1,\alpha) \geq A(t-1) \}}\\
&\leq \sum_{s_{j_t^*}=1}^{ \lceil \lambda (t-1)^ \alpha \rceil}\,\sum_{s_j=A(t-1)}^{\lceil \lambda (t-1)^ \alpha \rceil}\bs{1}_{\{n_j(t-1,\alpha)=s_j, \, n_{j^*_t}(t-1,\alpha)=s_{j_t^*}\}}\\ 
& \times \bs 1_{\{\overline{r}_{j}(t-1,\alpha) +c_j(t-1,\alpha)>\overline{r}_{j^*_t}(t-1,\alpha) +c_{j^*_t}(t-1,\alpha), \, j \notin J(t)\} }. 
\end{align*}
Event $\overline{r}_{j}(t-1,\alpha) +c_j(t-1,\alpha)>\overline{r}_{j^*_t}(t-1,\alpha) +c_{j^*_t}(t-1,\alpha)$ is true if at least one of the following events is true. 
\begin{align}
&\overline{r}_{j}(t-1,\alpha) \geq  M_j(t-1) + c_j(t-1,\alpha) \label{h1}\\
&\overline{r}_{j^*_t}(t-1,\alpha) \leq  M_{j^*_t}(t-1) - c_{j^*_t}(t-1,\alpha) \label{h}\\
&M_{j^*_t}(t-1) - M_j(t-1)< 2 c_j(t-1,\alpha).\label{h2}
\end{align}
Since the change in the mean reward for any arm within time-window $\{t-\tau(t\!-\!1, \alpha), \ldots, t\}$ is less than $\lceil \lambda (t-1)^\alpha \rceil \epsilon_T$ and $\mu_{j^*_t}(t)-\mu_j(t) >\sigma$, 
we have $M_{j^*_t}(t\!-\!1) - M_j(t\!-\!1)>t^{-\frac{\alpha}{3}}$. Furthermore,~\eqref{h2} does not hold if $n_j(t\!-\!1,\alpha) \geq A(t\!-\!1)$. 

Now applying the Chernoff-Hoeffding inequality~\cite[Theorem 1]{WH:63} to~(\ref{h1}-\ref{h}), we obtain
\begin{align*}
&\mathbb{P}( \overline{r}_{j}(t-1,\tau) \geq  M_j(t) + c_j(t-1,\tau) )\leq (t-1)^{-2(1+\alpha)}\\
&\mathbb{P}( \overline{r}_{j^*_t}(t-1,\tau) \leq  M_{j_t^*}(t) -  c_{j_t^*}(t-1,\tau) )\leq (t-1)^{-2(1+\alpha)}.
\end{align*}
It follows similarly to proof of Theorem~\ref{thm:sw-ucb-abrupt} that the expected value of the third term in~\eqref{g} satisfies
\begin{align}
\mathbb{E} \left[\sum_{t=N+1}^T \!\!\!\! \bs{1}_{\{ j_t=j \neq j^*_t,n_j(t-1,\alpha) \geq A(t-1), j \notin J(t) \}}\right] 
\leq \frac{(\lambda+1)^2\pi^2}{3}. \label{eq:third-term}
\end{align}

Therefore, from~\eqref{g},~\eqref{eq:second-term}, and~\eqref{eq:third-term}, we have
\begin{align*}
R(T) &\leq \sum_{j=1}^{N}\Delta_j \Big\{\frac{12(1+\alpha)}{3-\alpha}\lambda^{\frac{2\alpha}{3-\alpha}}[(1-\alpha)(G+1)]^{\frac{3-\alpha}{3-3\alpha}}\ln T\\
&+2G +\frac{(\lambda+1)^2\pi^2}{3}+1\Big\}+\sum_{t=1}^{T}\sigma \\
&\leq \sum_{j=1}^{N}\Delta_j \{ \frac{12(1+\alpha)}{3-\alpha}\lambda^{\frac{2\alpha}{3-\alpha}}[(1-\alpha)(G+1)]^{\frac{3-\alpha}{3-3\alpha}}\ln T\\
&+2G+\frac{(\lambda+1)^2\pi^2}{3}+1\Big\}+[2+\frac{2\lambda}{\alpha+1}T^{1+\alpha}]\epsilon_T\\
&+\frac{3}{3-\alpha}(T+1)^{\frac{3-\alpha}{3}}.
\end{align*}
Since $G\in O(T^{1-\alpha})$, we have $R^{\text{SW-UCB\#}}(T) \in O(T^{1-\frac{\alpha}{3}}\ln T)$.
\end{proof}


\section{Numerical Illustration}\label{sec:simulation}
In this section, we present simulation results for the SW-UCB\# and LM-DSEE algorithms in both abruptly-changing and slowly-varying environments.  For all the simulations, we consider a 10-armed bandit in which the reward at each arm is generated using Beta distribution. For the abruptly-changing environment, the breakpoints are introduced at time-instants where the next element of the sequence $\seqdef{\lfloor t^{\nu}\rfloor}{t \in \until{T}}$ is different from the current element. At each breakpoint, the mean rewards at each arm are randomly selected from the set $\{0.05, 0.12, 0.19, 0.26, 0.33, 0.39, 0.46, 0.53, 0.6, 0.9\}$. In the slowly-varying environment, the change in the mean reward at an arm is uniformly randomly sampled from the set $[-2T^{-{\kappa}}, 2T^{-{\kappa}}]$. For Algorithm~\ref{algo:lm-dsee}, we select $(a,b)$  equal to $(1,0.25)$ and $(20,1)$
 for abruptly-changing and slowly-varying environments, respectively. For Algorithm \ref{algo:sw-ucb}, we select $\lambda=12.3$ and $\lambda=4.3$ for abruptly-changing and slowly-varying environments, respectively.The parameters $\nu$ and $\kappa$ that describe characteristics of non-stationarity are varied to evaluate the performance of algorithms. Figs.~\ref{fig:dsee} and~\ref{fig:ucb} show that both SW-UCB\# and LM-DSEE are effective in non-stationary environments.

\begin{figure}
\centering
\includegraphics[width=0.47\linewidth, height=0.4\linewidth, keepaspectratio]{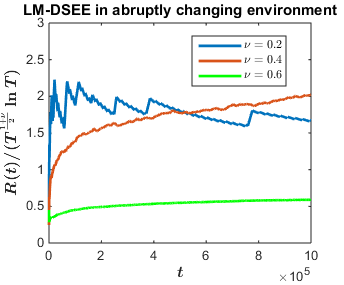}
\includegraphics[width=0.5\linewidth, height=0.4\linewidth, keepaspectratio]{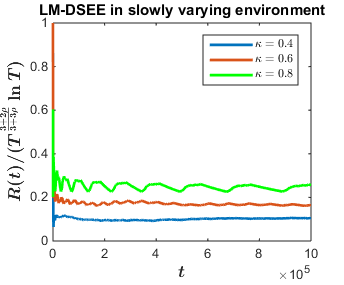} \\
\caption{The performance of the LM-DSEE algorithm in abruptly-changing and slowly-varying environments. \label{fig:dsee}}
\end{figure}

\begin{figure}
\centering
\includegraphics[width=0.5\linewidth, height=0.4\linewidth, keepaspectratio]{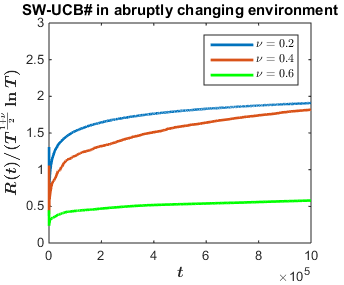}
\includegraphics[width=0.485\linewidth, height=0.4\linewidth, keepaspectratio]{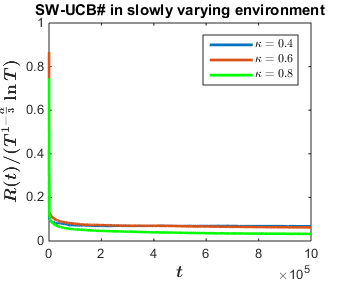}\\
\caption{The performance of the SW-UCB\# algorithm in abruptly-changing and slowly-varying environments. \label{fig:ucb}}
\end{figure}

In can be seen in Figs.~\ref{fig:dsee} and~\ref{fig:ucb} that for both algorithms in either of the environments, as expected, the ratio of the empirical regret to the order of the regret established in Sections~\ref{sec:LM-DSEE} and \ref{sec:SW-UCB} is upper bounded by a constant. The regret for the SW-UCB\# is relatively smoother than the regret for the LM-DSEE algorithm. The saw-tooth behavior of the regret for LM-DSEE is attributed to the fixed exploration-exploitation structure, wherein the regret certainly increases during the exploration epochs. 



While both the algorithms incur the same order of regret, compared with LM-DSEE, SW-UCB\# has a better leading constant. This illustrates the cost of constraining the algorithm to have a deterministic structure. On the other hand, this deterministic structure can be very useful, for example, in the context of planning trajectories for a mobile robot performing search using an MAB framework. 


\section{Conclusions and Future Directions}\label{sec:conclusions}
We studied the stochastic MAB problem in two classes of non-stationary environments and designed two novel algorithms, LM-DSEE and SW-UCB\# for these problems. We analyzed these algorithms for abruptly-changing  and slowly-varying environments, and characterized their performance in terms of expected cumulative regret. In particular, we showed that these algorithms incur sublinear expected cumulative regret, i.e., the time average of the regret asymptotically converges to zero. 

There are several possible avenues for future research. In this paper, we focused on a single decision-maker. Extensions of this work to multiple decision-makers is of significant interest. Implementation of these algorithms for robotic search and surveillance is an exciting direction to pursue. Finally, extension of the methodology developed in this paper to other classes on MAB problems such as the Markovian MAB problem and the restless MAB problem is also of interest.

\bibliographystyle{IEEEtran}
\bibliography{IEEEabrv,bandits,surveillance,mybib}

\end{document}